\documentclass{article} % For LaTeX2e
\usepackage{nips13submit_e,times}
\nipsfinalcopy
\usepackage{amsmath}
\usepackage{amsthm}
\usepackage{amsfonts}
\usepackage{algorithm}
\usepackage{graphicx}
\usepackage[numbers]{natbib}
\usepackage{bbm}     % for indicator function

\newcommand{\E}{\mathbb{E}}
\newcommand{\R}{\mathbb{R}}
\newcommand{\cP}{\mathcal{P}}
\newcommand{\cX}{\mathcal{X}}
\newcommand{\cU}{\mathcal{U}}
\newcommand{\cZ}{\mathcal{Z}}
\newcommand{\T}{^{\top}}

\DeclareMathOperator*{\argmax}{arg\,max}

\title{Scaling Up Robust MDPs by Reinforcement Learning}

\author{
Aviv Tamar\\
Electrical Engineering Department\\
The Technion - Israel Institute of Technology\\
Haifa, Israel 32000\\
\texttt{avivt@tx.technion.ac.il} \\
\And
Huan Xu \\
Mechanical Engineering Department \\
National University of Singapore \\
Singapore 117575, Singapore \\
\texttt{mpexuh@nus.edu.sg} \\
\And
Shie Mannor \\
Electrical Engineering Department\\
The Technion - Israel Institute of Technology\\
Haifa, Israel 32000\\
\texttt{shie@ee.technion.ac.il} \\
}

\newtheorem{theorem}{Theorem}
\newtheorem{assumption}[theorem]{Assumption}
\newtheorem{proposition}[theorem]{Proposition}
\newtheorem{corollary}[theorem]{Corollary}

\begin{document}

\maketitle

\begin{abstract}
We consider {\em large-scale} Markov decision processes (MDPs) with parameter uncertainty, under the robust MDP paradigm. Previous studies showed that robust MDPs, based on a minimax approach to handle uncertainty, can be solved using dynamic programming for {\em small to medium sized} problems. However, due to the ``curse of dimensionality'', MDPs that model real-life problems are typically prohibitively large for such approaches. In this work we employ a reinforcement learning approach to tackle this planning problem: we develop a  {\em robust approximate dynamic programming} method based on a projected fixed point equation to approximately solve large scale robust MDPs. We show that the proposed method provably succeeds under certain technical conditions, and demonstrate its effectiveness through simulation of an option pricing problem. To the best of our knowledge, this is the first attempt to scale up the robust MDPs paradigm.

\end{abstract}

\section{Introduction}
% MDP
Markov decision processes (MDPs) are standard models for solving sequential decision making problems in stochastic dynamic environments~\cite{Puterman1994,Bertsekas96}. Given the parameters, namely, transition probability and immediate reward, the strategy that achieves maximal expected accumulated reward is considered optimal. However, in practice,
these parameters are typically estimated from noisy data, or even worse,
they may change during the execution of a policy. It is thus not surprising that
  the actual performance of the chosen strategy can significantly differ from the model's prediction due to such {\em parameter uncertainty}~--~the deviation of the model parameters from the true ones (see experiments in \cite{MannorSimesterSunTsitsiklis07}).

% Robust MDP,
To mitigate   performance deviation due to parameter uncertainty, the robust MDP framework, initially proposed in~\cite{iyengar_robust_2005,NilimG05,Bagnell01}, is now a common method. In this context, it is assumed that the {\em uncertain} parameters can be any member of a known set
(termed the ``uncertainty set''), and solutions are ranked based on their performance under the (respective)
worst  parameter realizations. Under mild technical conditions, the optimal solution of a robust MDP can be solved using dynamic programming, at least for small to medium sized MDPs.

% Large scale
This paper considers planning in large scale robust MDPs, a setup largely untouched in literature.   It is widely known that, due to the ``curse of dimensionality'', practical problems modeled as MDPs often have prohibitively large state-spaces, under which dynamic programming becomes intractable. Many approximation schemes have been proposed to alleviate the curse of dimensionality of large scale MDPs, among them approximate dynamic programming (ADP) is a popular approach~\cite{Powell11}. ADP considers approximations of the optimal value function, for example, as a linear functional of some features of the state, that can be solved efficiently using a sampling based approach.   Inspired by the empirical success of ADP in a broad range of application domains involving large scale MDPs,   we adapt it to the robust MDP setup, and develop and analyze methods that handle large scale robust MDPs. From a high level, we indeed solve a planning problem via a reinforcement learning (RL) approach: while the robust MDP model, the parameters, and the uncertainty sets are all known, and hence the optimal solution is well defined, we still use an RL approach to approximately find the solution due to the scale of the problem~\cite{SuttonBarto98}.
Our specific contributions are the following:
\begin{enumerate}
\item A framework for approximate solution of large-scale robust MDPs.
\item Convergence proof for robust policy evaluation with linear function approximation.
\item A robust policy improvement algorithm with linear function approximation.
\item Application of the framework to the problem of option pricing.
\end{enumerate}

\section{Background}
We describe our problem formulation and some preliminaries from robust MDPs and ADP.
\subsection{Robust Markov Decision Processes}
For a discrete set $\cal{B}$, let $\mathcal M (\cal B)$ denote the set of probability measures on $\cal B$, and let $| \cal{B} |$ denote its cardinality.
A Markov Decision Process (MDP; \cite{Puterman1994}) is a tuple $\{ \cX, \cZ, \cU, P, r, \gamma \}$ where $\cX$ is a finite set of states, $\cZ$ is a (possibly empty) set of absorbing terminal states, and $\cU$ is a finite set of actions. Also, $r:\cX \times \cU\to \R$ is a deterministic and bounded reward function, $\gamma$ is a discount factor, and $P:\cX \times \cU\to \mathcal M (\cX \cup \cZ)$ denotes the probability distribution of next states, given the current state and action. We assume zero reward at terminal states.

A policy $\pi:\cX \to \mathcal M (\cU)$ maps each state to a probability distribution over the actions. The value of a state $x$ under policy $\pi$ is denoted $V^{\pi,P}(x)$ and represents the expected sum of discounted returns when starting from that state and executing $\pi,\quad$
\begin{equation*}
%$
V^{\pi,P}(x) = \E^{\pi,P}\left[ \left. \sum_{t=0}^{\infty} \gamma^{t} r(x_t, u_t)\right| x_0 = x\right],
%$
\end{equation*}
where $\E^{\pi,P}$ denotes expectation w.r.t. the state-action distribution induced by the transitions $P$ and the policy $\pi$. Note that for any terminal state $z\in \cZ$ and all $\pi$ and $P$ we have $V^{\pi,P}(z)=0$.

Typically in MDPs, one is interested in finding a policy that maximizes the value of certain (or all) states. When the state space is small enough, and all the parameters are known, efficient methods exist \cite{Puterman1994}. In practice, however, the state transition probabilities may not be exactly known. A widely-applied approach in this setting is  the Robust MDP (RMDP; \cite{NilimG05,iyengar_robust_2005}, also termed Ambiguous MDP). In this framework, the unknown transition probabilities are assumed to  lie in some \emph{known} uncertainty set. Such a set may be obtained, for example, from statistical confidence intervals when the transition probabilities are estimated from data.  Mathematically,
an RMDP is a tuple $\{ \cX, \cZ, \cU, \cP, r, \gamma \}$ where $\cX,\cZ,\cU,r,$ and $\gamma$ are as defined for MDPs. The uncertainty set $\cP$, where $\cP (x,u) \subset \mathcal M (\cX \cup \cZ)$, denotes a known uncertainty in the state transitions. Note that this definition implicitly assumes a \emph{rectangularity} of the uncertainty set \cite{iyengar_robust_2005}. Also note, that an RMDP is reduced to an MDP when there is no uncertainty, i.e., when $\cP(x,u)$ is a singleton for all $x$ and $u$.
%A terminal (absorbing) state $z$ is a state for which $\cP (z,u) = \I{x = z}$ and $r(z,u)=0$ for all $u\in \cU$. We let $\cZ$ denote the set of terminal states in the RMDP, and $\cZ$ is allowed to be an empty set. We also denote by $\cX \backslash \cZ$ the set of non-terminal states.
%Similarly, the state-action value function $Q^{\pi,P}$ for a state-action pair $x,u$ is
%\begin{equation*}
%Q^{\pi,P}(x,u) = \E^{\pi,P}\left[ \left. \sum_{t=0}^{\infty} \gamma^{t} r(x_t, u_t)\right| x_0 = x, u_0 = u\right].
%\end{equation*}
In robust MDPs, one is typically interested in maximizing the \emph{worst case} performance. Formally, we define the robust value function \cite{iyengar_robust_2005,NilimG05} for a policy $\pi$ as its worst-case value function
%over the uncertainty set $\cP$
\begin{equation*}
V^{\pi}(x) = \inf_{P\in \cP}V^{\pi,P}(x),
\end{equation*}
and we seek for the optimal robust value function
%\begin{equation*}
$
V^*(x) = \sup_{\pi}\left\{\inf_{P\in \cP}V^{\pi,P}(x)\right\}.
$
%\end{equation*}
In \cite{iyengar_robust_2005,NilimG05} it was shown that similarly to the regular value function, the robust value function is obtained by a deterministic policy, and satisfies a (robust) Bellman recursion of the form
\begin{equation*}
V^*(x) = \sup_{u\in \cU}\left\{ r(x,u) + \gamma \inf_{P\in\cP} \E^{P} \left[V^*(x') | x,u\right] \right\},
\end{equation*}
where $x'$ denotes the state following the state $x$ and action $u$. Thus, in the sequel we shall only consider deterministic policies, and write $\pi(x)$ as the action prescribed by policy $\pi$ at state $x$.

In \cite{iyengar_robust_2005}, a policy iteration algorithm was proposed for the robust MDP framework. This algorithm repeatedly improves a policy $\pi$ by choosing greedy actions with respect to $V^{\pi}$. The key step in this approach is therefore policy evaluation - calculating $V^{\pi}$, which satisfies
\begin{equation}\label{eq:rob_bellman_pi}
V^{\pi}(x) =  r(x,\pi (x)) + \gamma \inf_{P\in\cP} \E^{P} \left[ V^{\pi}(x') | x,\pi(x)\right].
\end{equation}

The non-linear equation \eqref{eq:rob_bellman_pi} may be solved for $V^{\pi}$ using an iterative method as follows. Let us first write \eqref{eq:rob_bellman_pi} in vector notation. For some $x$ and $u$ we define the operator $\sigma_{\cP (x,u)}:\R^{|\cX|} \to \R$ as
\begin{equation*}
\sigma_{\cP (x,u)}v \doteq \inf\left\{p\T v:p\in \cP (x,u)]\right\},
\end{equation*}
where $v\in \R^{|\cX|}$ and, slightly abusing notation, we ignore transitions to terminal states in $\cP (x,u)$. Also, for some policy $\pi$ let the operator $\sigma_{\pi}:\R^{|\cX|} \to \R^{|\cX|}$ be defined such that $\left\{\sigma_{\pi}v\right\}(x) \doteq \sigma_{\cP (x,\pi(x))}v$. Then \eqref{eq:rob_bellman_pi} may be written as %\begin{equation*}
$
V^{\pi} =  r^{\pi} + \gamma \sigma_{\pi}V^{\pi}.
$
%\end{equation*}
Let $T^{\pi}:\R^{|\cX|} \to \R^{|\cX|}$ denote the robust Bellman operator for a fixed policy, defined by
\begin{equation}\label{eq:T_def}
T^{\pi}v \doteq r^{\pi} + \gamma \sigma_{\pi}v.
\end{equation}
We see that $V^{\pi}$ is a fixed point of $T^{\pi}$, i.e., $V^{\pi}=T^{\pi}V^{\pi}$. Furthermore, since $T^{\pi}$ is known to be a contraction in the sup norm \cite{iyengar_robust_2005}, $V^{\pi}$ may be found by iteratively applying $T^{\pi}$ to some vector $v$.

\subsection{Projected Fixed Point Equation Methods}
For   MDPs, when the state space is large, dynamic programming methods become intractable, and one has to resort to an approximation procedure. A popular approach involves a projection of the value function onto a lower dimensional subspace by means of linear function approximation \cite{BT96}, and solving the solution of a \emph{projected} Bellman equation. We briefly review this approach.

Assume a regular MDP setting without uncertainty, where the Bellman equation \eqref{eq:rob_bellman_pi} for a fixed policy is reduced to $V^{\pi}(x) =  r(x,\pi (x)) + \gamma \E^{P} V^{\pi}(x')$. When the state space is large, calculating $V^{\pi}(x)$ for every $x$ in prohibitively computationally expensive, and a lower dimensional approximation of $V^{\pi}$ is sought.
Consider the linear approximation given by a weighted sum of features
\begin{equation*}
\tilde{V}^{\pi}(x) = \phi(x)\T w, \quad x\in \cX,
\end{equation*}
where $\phi(x)\in \R^k$,  $k<|\cX|$ contains the features of state $x$ and $w\in \R^k$ are the approximation weights. Let $\Phi \in \R^{|\cX| \times k}$ denote a matrix with the feature vectors in its rows.
%, and let $d\in \R^{|\cX|}$ with $d>0$.
A popular approach for finding $w$ is by solving the \emph{projected Bellman equation} \cite{Ber2012DynamicProgramming}, given by
\begin{equation}\label{eq:proj_eq}
\tilde{V}^{\pi} = \Pi T^{\pi} \tilde{V}^{\pi},
\end{equation}
where $\Pi$ is a projection operator onto the subspace spanned by $\Phi$ with respect to a $d$-weighted Euclidean norm. At this point we only assume that $d\in \R^{|\cX|}$ is positive. Since there is no uncertainty, $T^{\pi}$ here is a linear mapping, and Equation \eqref{eq:proj_eq} may be written in matrix form as follows
\begin{equation}\label{eq:proj_eq_mat}
\Phi\T D \Phi w = \Phi\T D r + \Phi\T D P^\pi \Phi w,
\end{equation}
where $D=\textrm{diag}(d)$, and $P^{\pi} \in \R^{|\cX| \times |\cX|}$ is the Markov transition matrix induced by policy $\pi$. Given $\Phi\T D \Phi$, $\Phi\T D r$, and $\Phi\T D P^\pi \Phi$, Eq. \eqref{eq:proj_eq_mat} may be solved for $w$ either by matrix inversion \cite{boyan2002technical}, or iteratively (known as Projected Value Iteration; PVI; \cite{Ber2012DynamicProgramming})
\begin{equation}\label{eq:proj_eq_iter_sol}
w_{k+1} = \left( \Phi\T D \Phi \right)^{-1} \left( \Phi\T D r + \gamma \Phi\T D P^\pi \Phi w_k \right).
\end{equation}
When $d$ corresponds to the steady state distribution over states for policy $\pi$, the iterative procedure in \eqref{eq:proj_eq_iter_sol} can be shown to converge using contraction properties of $\Pi T^{\pi} $ \cite{Ber2012DynamicProgramming}. For a large state space, the terms in \eqref{eq:proj_eq_iter_sol} cannot be calculated explicitly. However, the strength of this approach is that these terms may be sampled efficiently, using trajectories from the MDP \cite{Ber2012DynamicProgramming}.

Recall that our ultimate goal is policy improvement. For a regular MDP, the policy evaluation procedure described above may be combined with a policy improvement step using Least Squares Policy Iteration (LSPI; \cite{lagoudakis2003least}), which extends policy iteration to the function approximation setting.
%\begin{equation}\label{eq:sampling_estimates}
%\Phi\T D \Phi \sim \frac{1}{k}\sum_{t=1}^{k}\phi(x_t)\phi(x_t)\T, \quad  \Phi\T D r \sim \frac{1}{k}\sum_{t=1}^{k}\phi(x_t)r(x_t,u_t)\T, \quad \Phi\T D P^\pi \Phi \sim \frac{1}{k}\sum_{t=1}^{k}\phi(x_t)\phi(x_{t+1})\T.
%\end{equation}
%This procedure of policy evaluation is known as Least Squares Policy Evaluation (LSPE; \cite{Ber2012DynamicProgramming}).

\section{Robust Policy Evaluation}\label{sec:RADP}
In this section we propose an extension of ADP to the robust setting. We do this as follows. First, we consider policy evaluation, and extend the projected fixed point equation \eqref{eq:proj_eq} to the robust case, with the robust $T^\pi$ operator as defined in \eqref{eq:T_def}. We discuss the conditions under which this equation has a solution, and how it may be obtained. We then propose a sampling based procedure to solve the equation for large state spaces, and prove its convergence. Finally, in Section \ref{sec:ARPI}, we will use our policy evaluation procedure as part of a policy improvement algorithm in the spirit of LSPI \cite{lagoudakis2003least}, for obtaining an (approximately) optimal robust policy.
\subsection{A Projected Fixed Point Equation}
Throughout this section we consider a fixed policy $\pi$. For some positive $d$, let the projection operator $\Pi$ be defined as above.
Consider the following \emph{projected robust Bellman equation} for a fixed policy
\begin{equation}\label{eq:rob_proj_eq}
\tilde{V}^{\pi} = \Pi T^{\pi} \tilde{V}^{\pi}.
\end{equation}
Note that here, as opposed to \eqref{eq:proj_eq}, $T^{\pi}$ is not necessarily linear, and hence it is not clear whether Eq. \eqref{eq:rob_proj_eq} has a solution at all.
We now show that under suitable conditions the operator $\Pi T^{\pi}$ is a contraction and Equation \eqref{eq:rob_proj_eq} has a \emph{unique} solution. We consider two different cases, depending on the existence of terminal states $\mathcal{Z}$. Let $\hat{P}:\cX \to \mathcal M (\cX \cup \cZ)$ represent some given state transitions probabilities. Slightly abusing notation, we let $\hat{P}(x_t=j)$ denote the probability that the state at time $t$ is $j$, given that the states evolve according to a Markov chain with transitions $\hat{P}$. In the sequel, $\hat{P}$ will be used to represent the \emph{exploration} policy of the MDP in an offline learning setting. We make the following assumption on $\hat{P}$, which also defines the projection weights $d$.
\begin{assumption}\label{ass:d_def}
Either $\cZ = \emptyset$, and there exists positive numbers $d_j$ such that
\begin{equation*}
d_j = \lim_{t\to\infty} \hat{P}(x_t=j|x_0=i) \quad \forall i,j\in \cX,
\end{equation*}
or $\cZ \neq \emptyset$, and the policy underlying $\hat{P}$ is \emph{proper} \cite{Ber2012DynamicProgramming}, that is, for $\bar{t}=|\cX |$
\begin{equation*}
\hat{P}(x_{\bar{t}}\in \cZ|x_0=i)>0 \quad \forall i\in \cX,
\end{equation*}
and all states have a positive probability of being visited. In this case we let
\begin{equation*}
d_j = \sum_{t=0}^{\infty} \hat{P}(x_t=j) \quad \forall j\in \cX.
\end{equation*}
\end{assumption}
The following key assumption relates the transitions of the exploration policy and the (uncertain) transitions of the policy under evaluation. We further discuss its significance in Section \ref{sec:remarks}.
\begin{assumption}\label{ass:P_lt_hatP}
There exists $\beta \in (0,1)$ such that
%\begin{equation*}
$
\gamma P(x'|x,\pi(x)) \leq \beta \hat{P}(x'|x,\pi(x)), \quad \forall P\in \cP, x\in \cX, x'\in \cX.
$
%\end{equation*}
\end{assumption}
Let $\| \cdot \|_d$ denote the $d$-weighted Euclidean norm, which is well-defined due to Assumption \ref{ass:d_def}. Our key insight is the following proposition, which shows that under Assumption \ref{ass:P_lt_hatP}, the robust Bellman operator is a $\beta$-contraction in $\| \cdot \|_d$.
\begin{proposition}
Let Assumptions \ref{ass:d_def} and \ref{ass:P_lt_hatP} hold. Then
$
\| T^{\pi}y - T^{\pi}z \|_d \leq \beta \| y - z \|_d
$
for all $y,z \in \R^{|\cX|}$
\end{proposition}
\begin{proof}
Fix $x \in \cX$, and assume that $T^{\pi}y(x) \geq T^{\pi}z(x)$. Choose some $\epsilon > 0$, and $P_x \in \cP$ such that
\begin{equation}\label{eq:proof_1}
\E^{P_x} \left[ \left. z(x') \right| x,\pi(x) \right] \leq \inf_{P\in\cP} \E^{P} \left[ \left. z(x') \right| x,\pi(x) \right] + \epsilon.
\end{equation}
Also, note that by definition
\begin{equation}\label{eq:proof_2}
\inf_{P\in\cP} \E^{P} \left[ \left. y(x') \right| x,\pi(x) \right] \leq  \E^{P_x} \left[ \left. y(x') \right| x,\pi(x) \right].
\end{equation}
Now, we have
\begin{equation*}
\begin{split}
0   \leq T^{\pi}y(x) - T^{\pi}z(x)
%    &= \gamma \inf_{P\in\cP} \E^{P} y(x') - \gamma \inf_{P\in\cP} \E^{P} z(x')\\
    &\leq (\gamma \E^{P_x} \left[ \left. y(x') \right| x,\pi(x) \right] ) - (\gamma \E^{P_x} \left[ \left. z(x') \right| x,\pi(x) \right] - \gamma \epsilon) \\
    &= \gamma \E^{P_x} \left[ \left. y(x') - z(x') \right| x,\pi(x) \right] +\gamma \epsilon \\
%    &\leq \gamma \E^{P_x} \left[ \left. \left| y(x') - z(x') \right| \: \right| x,\pi(x) \right] +\gamma \epsilon \\
    &\leq \beta \E^{\hat{P}} \left[ \left. \left| y(x') - z(x') \right| \: \right| x,\pi(x) \right] +\gamma \epsilon,
\end{split}
\end{equation*}
where the second inequality is by \eqref{eq:proof_1} and \eqref{eq:proof_2}, and the last inequality is by Assumption \ref{ass:P_lt_hatP}.
Conversely, if $T^{\pi}z(x) \geq T^{\pi}y(x)$, following the same procedure we obtain
%\begin{equation*}
$
0 \leq T^{\pi}z(x) - T^{\pi}y(x) \leq \beta \E^{\hat{P}} \left[ \left. \left| y(x') - z(x') \right| \: \right| x,\pi(x) \right]+\gamma \epsilon,
$
%\end{equation*}
and we therefore conclude that
%\begin{equation*}
$
 \left| T^{\pi}y(x) - T^{\pi}z(x) \right| \leq \beta \E^{\hat{P}} \left[ \left. \left| y(x') - z(x') \right|  \: \right| x,\pi(x) \right] +\gamma \epsilon.
$
%\end{equation*}
Since $\epsilon$ was arbitrary, we have that
%\begin{equation*}\label{eq:proof_3}
$
 \left| T^{\pi}y(x) - T^{\pi}z(x) \right| \leq \beta \E^{\hat{P}} \left[ \left. \left| y(x') - z(x') \right| \: \right| x,\pi(x) \right]
$
%\end{equation*}
for all $x$, and therefore
\begin{equation*}
\begin{split}
 \left\| T^{\pi}y - T^{\pi}z \right\|_d &\leq \beta \left\|\hat{P} \left| y - z\right| \right\|_d
 \leq \beta \left\| y - z \right\|_d,
\end{split}
\end{equation*}
where in last equality we used the well-known result that the state transition matrix $\hat{P}$  is contracting in the $d$-weighted Euclidean norm \cite{Ber2012DynamicProgramming}.
\end{proof}
Since the projection operator $\Pi$ is known to be non-expansive in the $d$-weighted norm \cite{Ber2012DynamicProgramming}, we have the following corollary.
\begin{corollary}\label{cor:Pi_T_Contraction}
Let Assumptions \ref{ass:d_def} and \ref{ass:P_lt_hatP} hold. Then the projected robust Bellman operator $\Pi T^{\pi}$ is a $\beta$-contraction in the $d$-weighted Euclidean norm.
\end{corollary}
The contraction property in Corollary \ref{cor:Pi_T_Contraction} guarantees an error bound for the fixed point approximation on the order of $1/(1-\beta)$ \cite{Ber2012DynamicProgramming}. It also suggests a straightforward procedure for solving Equation \eqref{eq:rob_proj_eq} which we describe next.

\subsection{Robust Projected Value Iteration}
A natural method for for solving Equation \eqref{eq:rob_proj_eq} is the robust equivalent of PVI
\begin{equation}\label{eq:RPVI}
\Phi w_{k+1} = \Pi T^{\pi} \left( \Phi w_{k} \right).
\end{equation}
Corollary \ref{cor:Pi_T_Contraction} guarantees that the iterates of \eqref{eq:RPVI} converge to the fixed point of $\Pi T^{\pi}$. The algorithm \eqref{eq:RPVI} may be written explicitly in matrix form (see \cite{Ber2012DynamicProgramming}) as
\begin{equation}\label{eq:rob_proj_eq_iter_sol}
w_{k+1} = \left( \Phi\T D \Phi \right)^{-1} \left( \Phi\T D r + \gamma \Phi\T D \sigma_{\pi}( \Phi w_k ) \right).
\end{equation}
We refer to the algorithm in \eqref{eq:rob_proj_eq_iter_sol} as \emph{robust projected value iteration} (RPVI).
%Note that in the case of no uncertainty RPVI is reduced to PVI \eqref{eq:proj_eq_iter_sol}.
Note that a matrix inversion approach would not be applicable here, as \eqref{eq:rob_proj_eq_iter_sol} is not linear due to  non-linearity of  $\sigma_{\pi}(\cdot)$.

For a large state space, computing the terms in \eqref{eq:rob_proj_eq_iter_sol} exactly is intractable. For this case we propose a sampling procedure for estimating these terms, as described next.

\subsection{A Sampling Based Approach}
When the state space is too large for the terms in Equation \eqref{eq:rob_proj_eq} to be computed exactly, one may resort to a sampling based procedure. This approach is popular in the RL and ADP literature, and has been used successfully on problems with very large state spaces \cite{Powell11}. Here, we describe how it may be applied for the robust MDP setting.

Assume that we have obtained a long trajectory from an MDP with transition probabilities $\hat{P}$, while following policy $\pi$. We denote this data by $x_0,u_0,r_0,x_1,u_1,r_1,\dots,x_N,u_N,r_N$. A very useful property of the terms in \eqref{eq:rob_proj_eq_iter_sol} is that they may be estimated from the data by\footnote{These estimates are for the case $\cZ = \emptyset$ in Assumption \ref{ass:d_def}. Modifying these estimates for the case $\cZ \neq \emptyset$ is straightforward, along the lines of Chapter 7.1 in \cite{Ber2012DynamicProgramming}.}
\begin{equation*}
\Phi\T D \Phi \sim \frac{1}{N}\sum_{t=0}^{N-1}\phi(x_t)\phi(x_t)\T, \quad  \Phi\T D r \sim \frac{1}{N}\sum_{t=0}^{N-1}\phi(x_t)r(x_t,u_t)\T,
\end{equation*}
and
\begin{equation}\label{eq:sampled_sigma}
\Phi\T D \sigma_{\pi}( \Phi w_k ) \sim \frac{1}{N}\sum_{t=0}^{N-1}\phi(x_t) \sigma_{\cP (x_t,u_t)}(\Phi w_k).
\end{equation}
Using the law of large numbers, it may be proved \footnote{The proof is similar to the case without uncertainty, detailed in \cite{Ber2012DynamicProgramming}.} that these estimates converge with probability 1 to their respective terms in \eqref{eq:rob_proj_eq_iter_sol} as $N\to\infty$. Together with Corollary \ref{cor:Pi_T_Contraction} we have the following convergence result. The proof is straightforward and omitted.
\begin{proposition}
Let Assumptions \ref{ass:d_def} and \ref{ass:P_lt_hatP} hold. Consider the RPVI algorithm with the terms in \eqref{eq:rob_proj_eq_iter_sol} replaced by their sampled counterparts \eqref{eq:sampled_sigma}. Then as $N\to\infty$ and $k\to\infty$, $w_k$ converges with probability 1 to $w^*$, and $\Phi w^*$ is the unique solution of \eqref{eq:rob_proj_eq}.
\end{proposition}

In Eq. \eqref{eq:sampled_sigma}, the calculation of $\sigma_{\cP (x_t,u_t)}(\Phi w_k)$ requires a model, and, depending on the uncertainty set and state transitions, may be computationally demanding. One very natural class of models, proposed in \cite{iyengar_robust_2005,NilimG05}, is constructed from empirical state transitions $x_t \to x_{t+1}$, and the uncertainty set corresponds to confidence regions associated with probability density estimation. In these studies, efficient methods for performing the minimization in $\sigma_{\cP (x_t,u_t)}$ were suggested. In the case of binary transitions, as in our option pricing example of Section \ref{sec:options}, performing the minimization is trivial.

%Let $\tilde{p}(x'|x) = \frac{\sum_t \I{x_t = x,x_{t+1}=x'}}{\sum_t \I{x_t = x}}$ denote the empirical transition frequencies, and let ${\cal B }(x) = \{x'\in \cX: \tilde{p}(x'|x)>0 \}$. Consider the uncertainty set
%\begin{equation*}
%\cP_\eta (x,u) = \left\{ p \in \mathcal M ({\cal B} (x) ): D(p \| \tilde{p}) < \eta \right\},
%\end{equation*}
%where $D(p \| \tilde{p})$ is the Kullback-Leibler distance between $p$ and $\tilde{p}$. In \cite{iyengar_robust_2005}, the threshold $\eta$ is related to a confidence region of the distance between the true distribution $\hat{P}$ and the empirical frequencies $\tilde{p}$, and an efficient method for calculating the minimization in $\sigma_{\cP (x_t,u_t)}$ is suggested.
%Thus, for this setting, the robust LSPE procedure may be computed efficiently.

%The law of large numbers together with Corollary \ref{cor:Pi_T_Contraction} suggest that as $N\to\infty$ and $k\to\infty$, robust LSPE converges to a solution of \eqref{eq:rob_proj_eq}.

\subsection{Some Remarks on Assumption \ref{ass:P_lt_hatP}}\label{sec:remarks}
Assumption \ref{ass:P_lt_hatP} may appear quite restrictive, especially when the discount factor $\gamma$ approaches 1. At present, we are not aware of a relaxation that will work for general features. However, we emphasize that the inequality in Assumption \ref{ass:P_lt_hatP} is not required for transitions to a terminal state. This is significant, for example, in optimal stopping problems. There, if Assumption \ref{ass:P_lt_hatP} holds for an exploration policy that never stops, it can be shown to hold for all policies; we discuss this in more detail in Section \ref{sec:options}.

One may question whether this exception for terminal states is due to the fact that their value is not approximated, and whether we can cope with states for which the assumption does not hold by not approximating their values. Unfortunately, this is not the case, as we show in the supplementary material that even if Assumption \ref{ass:P_lt_hatP} fails for a single state, and for that state there is no approximation, iteratively applying $\Pi T^{\pi}$ may diverge.

We note that a similar difficulty arises in off-policy RL \cite{bertsekas_yu_2009,sutton_fast_2009} (in fact, our Assumption \ref{ass:P_lt_hatP} is similar to an assumption in \cite{bertsekas_yu_2009}) , where some algorithms are shown to converge to a solution of \eqref{eq:proj_eq} even when $\Pi T^{\pi}$ is not a contraction \cite{yu_convergence_2010,sutton_fast_2009}. However, in these cases not much can be said about the solution itself, and we therefore do not pursue such an approach here.

Finally, we note that for averager type function approximation \cite{gordon1995stable}, $\Pi$ contracts in the sup norm, and since $T^{\pi}$ also contracts in the sup norm \cite{iyengar_robust_2005}, $\Pi T^{\pi}$ contracts regardless of Assumption \ref{ass:P_lt_hatP}.
%One may question whether Assumption 1 We show in the supplementary material, that even if Assumption 1 fails for a single state $\tilde{x}$, and for that state there is no approximation - i.e., there is a feature $\tilde{\phi}(x) = \I {x=\tilde{x}}$ that is orthogonal to all other features, iteratively applying $\Pi T^{\pi}$ may diverge.

\section{Robust Approximate Policy Iteration}\label{sec:ARPI}
In this section we propose a policy improvement algorithm, driven by the RPVI method of the previous section.
%in the spirit of LSPI, for robust MDPs. Our algorithm is

We begin by introducing the state-action value function $Q^{\pi}(x,u)$
\begin{equation*}
Q^{\pi}(x,u) = \inf_{P\in \cP} \E^{\pi,P}\left[ \left. \sum_{t=0}^{\infty} \gamma^{t} r(x_t, u_t)\right| x_0 = x, u_0 = u\right],
\end{equation*}
which is more convenient for applying the optimization step of policy iteration than $V^{\pi}(x)$. Again, we assume linear function approximation of the form
%\begin{equation*}
$
\tilde{Q}^{\pi}(x,u) = \phi(x,u)\T w,
$
%\end{equation*}
where $\phi(x,u)\in \R^k$ is a state-action feature vector and $w\in \R^k$ is a parameter vector. Note that $Q^{\pi}(x,u)$ may be seen as the value function of an equivalent RMDP with states in $\cX \times \cU$, therefore the policy evaluation algorithm of Section \ref{sec:RADP} applies. Also, note that given some $w$, a greedy policy $\pi^*_w(x)$ at state $x$ with respect to that approximation may be computed by
\begin{equation*}
\pi^*_w(x) = \argmax_u \phi(x,u)\T w,
\end{equation*}
and we write
%\begin{equation*}
$
\phi^*_w(x) = \phi(x,\pi^*_w(x)),
$
%\end{equation*}
and let $\Phi^*_w$ denote a matrix with $\phi^*_w(x)$ in its rows.
%and we let $\tilde{V}^*_w \in \R^{| \cX |}$ be defined such that
%\begin{equation*}
%\tilde{V}^*_w(x) = \max_u \phi(x,u)\T w.
%\end{equation*}

The Approximate Robust Policy Iteration (ARPI) algorithm is initialized with an arbitrary parameter vector $w_0$. At iteration $i+1$, we estimate the parameter $w_{i+1}$ of the \emph{greedy} policy with respect to $w_i$ as follows. We first initialize $\theta_0\in \R^k$ to some arbitrary value, and then iterate on $\theta$:
\begin{equation}\label{eq:theta_iterate}
\theta_{j+1} = \left( \Phi\T D \Phi \right)^{-1} \left( \Phi\T D r + \gamma \Phi\T D \sigma_{\pi}( \Phi^*_{w_i} \theta_j ) \right),
\end{equation}
where the terms in \eqref{eq:theta_iterate} are estimated from data (cf. \eqref{eq:rob_proj_eq_iter_sol}) according to
%\begin{equation*}
$
\Phi\T D \Phi \sim \frac{1}{N}\sum_{t=0}^{N-1}\phi(x_t,u_t)\phi(x_t,u_t)\T, \quad
\Phi\T D r \sim \frac{1}{N}\sum_{t=0}^{N-1}\phi(x_t,u_t)r(x_t,u_t)\T,
%\end{equation*}
$ and
$
%\begin{equation*}
\Phi\T D \sigma_{\pi}( \Phi^*_{w_i} \theta_j ) \sim \frac{1}{N}\sum_{t=0}^{N-1}\phi(x_t,u_t) \sigma_{\cP (x_t,u_t)}(\Phi^*_{w_i} \theta_j).
%\end{equation*}
$
After $\theta$ has converged, we set $w_{i+1}$ to its final value.

For comparison, in regular LSPI \cite{lagoudakis2003least} the iteration on $\theta$ is not needed, as the policy evaluation equation \eqref{eq:proj_eq} is linear, and may be solved using a least squares approach (LSTD; \cite{boyan2002technical}). Computationally, the contraction property of Corollary \ref{cor:Pi_T_Contraction} guarantees a linear convergence rate for the $\theta$ iteration, therefore the addition of this step should not impact performance significantly. Also, note that the computation of $\Phi\T D \Phi$ and $\Phi\T D r$ only needs to be done once.

\section{Applications}\label{sec:options}
In this section we discuss applications of robust ADP. We first consider optimal stopping problems, a subclass of MDPs, for which we can show that Assumption \ref{ass:P_lt_hatP} may be satisfied broadly. We then present an empirical evaluation on the problem of option pricing -- a finite horizon continuous state space optimal stopping problem, for which an exact solution is intractable.

\subsection{Optimal Stopping Problems}\label{sec:optimal_stopping}
An optimal stopping problem is an RMDP where the only choice is when to terminate the process. Formally, the action set is binary $\cU=\{0,1\}$, and executing $u=1$ from any state always transitions to a terminal state with probability 1 (and no uncertainty). Let $\hat{\pi}$ denote a policy that never chooses to terminate, i.e., $\hat{\pi}(x)=0$, $\forall x$. We now show that if Assumption \ref{ass:P_lt_hatP} is satisfied for $\hat{\pi}$, then it is immediately satisfied for all other policies. The proof is in the supplementary material.
\begin{proposition}\label{prop:OSprop}
Consider an optimal stopping problem, and let Assumption \ref{ass:P_lt_hatP} hold for $\hat{\pi}$. Then, for every policy $\pi$ we have \begin{equation*}\label{eq:OSprop1}
\gamma P(x'|x,\pi(x)) \leq \beta \hat{P}(x'|x,\pi(x)), \quad \forall P\in \cP, x\in \cX, x'\in \cX.
\end{equation*}
and
\begin{equation*}\label{eq:OSprop2}
\gamma P(x',\pi(x')|x,\pi(x)) \leq \beta \hat{P}(x',\pi(x')|x,\pi(x)), \quad \forall P\in \cP, x\in \cX, x'\in \cX.
\end{equation*}
\end{proposition}

\subsection{Option Pricing}\label{sec:option_pricing}
In this section we apply ARPI to the problem of pricing American-style options. An American-style put option is a contract which gives the owner the right, but not the obligation, to sell an asset at a specified strike price $K$ on or before some maturity time $T$. Letting $x_t$ denote the price (state) of the asset at time $t\leq T$, the immediate payoff of executing a put option at that time is therefore $\max\left( 0, K-x_t \right)$. Assuming Markov state transitions, an optimal execution policy may be found by solving a finite horizon optimal stopping problem, and the expected profit under that policy is termed the `fair' price of the option. Since the state space is typically continuous, an exact solution is infeasible, calling for approximate, sampling based techniques. Previous studies \cite{tsitsiklis2001options,li2009options} have proposed RL solutions for this task, and shown their utility. Here we extend this approach.

One challenge of option pricing is that the underlying model is never truly known, but instead we can only access historical data in the form of state trajectories (e.g., stock prices over time). Naturally, uncertainty in the options value as predicted from this data should reflect in its price. Here, we propose to price the option according to its \emph{robust} value, thereby treating uncertainty in a well-founded manner.

We now show how the option pricing problem may be formulated as an optimal stopping RMDP, and then present our empirical results of applying the ARPI algorithm to the problem.

\subsubsection{An RMDP Formulation}
The option pricing problem may be formulated as an RMDP as follows. To account for the finite horizon, we include time explicitly in the state, thus, the state at time $t$ is $\{ x_t, t \}$. The action set is binary, where $1$ stands for executing the option and $0$ for continuing to hold it. Once an option is executed, or when $t=T$, a transition to a terminal state takes place. Otherwise, the state transitions to $\{ x_{t+1}, t+1 \}$ where $x_{t+1}$ is determined by a stochastic kernel $\hat{P} (x'|x,t)$.
The reward for executing $u=1$ at state $x$ is $g(x) \doteq \max\left( 0, K-x \right)$ and zero otherwise.

Note that the state-action values for execution is known in advance, for we have $Q(\{x,t\},u=1) = g(x)$ by definition; therefore, we only need to estimate the value of continuation. We use linear function approximation $\tilde{Q}^{\pi}(\{x,t\},u=0) = \phi(\{x,t\})\T w$, and the ARPI update equation \eqref{eq:theta_iterate} in this case may be written as
%\begin{equation*}\label{eq:theta_iterate_options}
$
\theta_{j+1} = \left( \Phi\T D \Phi \right)^{-1} \left( \gamma \Phi\T D \sigma_{\pi}( \nu ) \right),
$
%\end{equation*}
where
%\begin{equation*}
%$
%\nu(x,t)=\begin{cases}
%g(x), & \textrm{if }g(x)>\phi(\{x,t\})\T w_i \\
%\phi(\{x,t\})\T \theta_j, & \textrm{else}
%\end{cases}.
%$
%\end{equation*}
$\nu(x,t)$ equals $g(x)$ if $g(x)>\phi(\{x,t\})\T w_i$, and equals $\phi(\{x,t\})\T \theta_j$ otherwise.
By Proposition \ref{prop:OSprop}, if the trajectories are obtained by a policy that never chooses to terminate, ARPI may be used safely as each policy evaluation step is guaranteed to converge.

\subsubsection{Results}
We focus on in-the-money options, where $K$ is equal to the initial price $x_0$. Our price fluctuation model follows a Bernoulli distribution \cite{cox1979option},
\begin{equation*}
x_{t+1} = \begin{cases}
f_{u} x_t, & \textrm{w.p. }p \\
f_{d} x_t, & \textrm{w.p. }1-p
\end{cases},
\end{equation*}
where the up and down factors, $f_u$ and $f_d$, are constant. Our empirical evaluation proceeds as follows.
In each experiment, we generate $N_{data}$ trajectories of length $T$ from the true model. From these trajectories we form the maximum likelihood estimate of the up probability $\hat{p}$, and the $100(1 - \alpha)\%$ confidence intervals $\hat{p}_{-}$ and $\hat{p}_{+}$ using the Clopper-Pearson method \cite{clopper1934use}, which constructs our uncertain model $M_{robust}$. We also build a model without uncertainty $M_{nominal}$ by setting $\hat{p}_{-}=\hat{p}_{+}=\hat{p}$. Using $\hat{p}$, we then simulate $N_{sim}$ trajectories of length $T$ (this corresponds to a policy that never executes the option), where $x_0=K+\epsilon$, and $\epsilon$ is uniformly distributed in $[-\delta,\delta]$. These trajectories are used as input data for the ARPI algorithm of Section \ref{sec:ARPI}.

For our linear function approximation we chose 2-dimensional (for $x$ and $t$) radial basis function (RBF) features. In comparison, \cite{li2009options} used Laguerre polynomials for $x$ and several monotone functions for $t$. We initially experimented with these features as well, but then opted for the RBF's, which displayed significantly better performance. We attribute this performance improvement to the non-separable (in $x$ and $t$) nature of the value function, a property that is not captured by the representation of \cite{li2009options}.

Let $\pi_{robust}$ and $\pi_{nominal}$ denote the policies found by ARPI using $M_{robust}$ and $M_{nominal}$, respectively. We evaluate the performance of $\pi_{robust}$ and $\pi_{nominal}$ using $N_{test}$ trajectories obtained from the \emph{true} model. In Figure \ref{fig:results} we compare the average $p-$percentiles (averaged over 200 independent experiments) of the total reward obtained by $\pi_{robust}$ and $\pi_{nominal}$, for different values of $\alpha$ and $N_{data}$. As expected, the robust policy gains higher payoff in the lower percentiles, while sacrificing payoff in higher percentiles, and displays a risk-averse behavior. The effect is proportional to the uncertainty, controlled by $\alpha$ and $N_{data}$.

The parameters for the experiments were chosen to balance the different factors in the problem, and are provided in the supplementary material, as well as Matlab code for reproducing these results.

\begin{figure}[h]
\begin{center}
\includegraphics[scale=0.33, trim=20 0 20 0, clip=true]{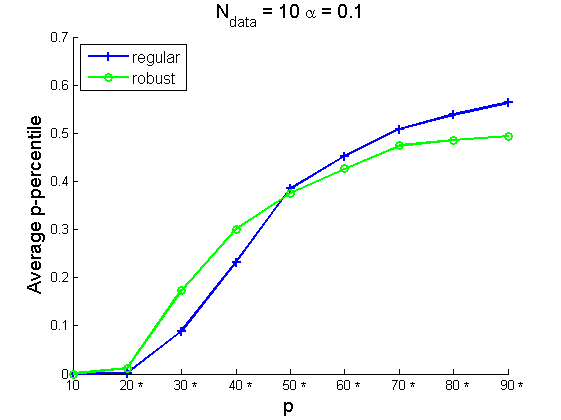}
\includegraphics[scale=0.33, trim=20 0 20 0, clip=true]{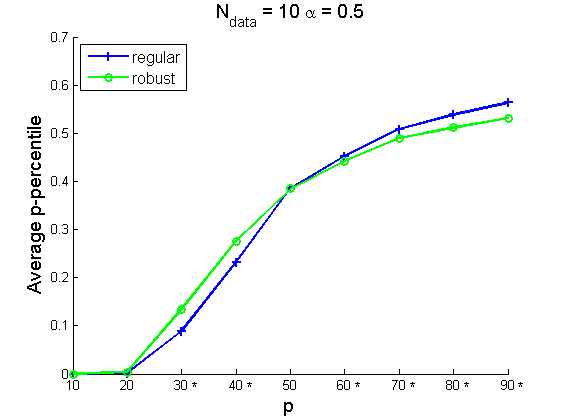}
\includegraphics[scale=0.33, trim=20 0 20 0, clip=true]{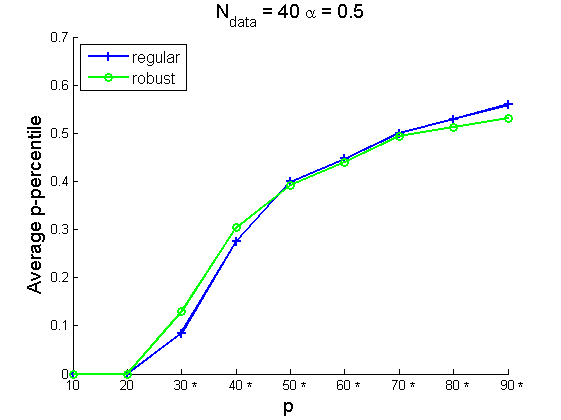}
\end{center}
\caption{\label{fig:results}Performance of robust vs. regular policies. The $p-$percentiles of the total reward, averaged over 200 independent runs of the experiment, are shown for different values of $\alpha$ and $N_{data}$. Percentiles for which the difference in performance is statistically significant (according to a paired t-test, $p<0.05$) are marked by an asterisk.}
\end{figure}

\section{Conclusion}
This work presented a novel framework for solving {\em large-scale} robust Markov decision processes. To the best of our knowledge, such problems are beyond the capabilities of previous studies, which focused on exact solutions and hence suffer from the ``curse of dimensionality''. Our approach to tackling the planning problem is through reinforcement learning methods: we reduce the dimensionality of the robust value function using linear function approximation, and employ an iterative sampling based procedure to learn the approximation weights. We presented both formal guarantees and empirical evidence to the usefulness of our approach in general robust MDPs, and optimal stopping problems in particular.

\bibliographystyle{abbrv}
\bibliography{RobustArxiv2013}

\begin{thebibliography}{10}

\bibitem{Bagnell01}
A.~Bagnell, A.~Ng, and J.~Schneider.
\newblock Solving uncertain {M}arkov decision problems.
\newblock Technical Report CMU-RI-TR-01-25, Carnegie Mellon University, August
  2001.

\bibitem{Ber2012DynamicProgramming}
D.~P. Bertsekas.
\newblock {\em Dynamic Programming and Optimal Control, Vol II}.
\newblock Athena Scientific, fourth edition, 2012.

\bibitem{Bertsekas96}
D.~P. Bertsekas and J.~N. Tsitsiklis.
\newblock {\em Neuro-Dynamic Programming}.
\newblock Athena Scientific, 1996.

\bibitem{BT96}
D.~P. Bertsekas and J.~N. Tsitsiklis.
\newblock {\em Neuro-Dynamic Programming}.
\newblock Athena Scientific, 1996.

\bibitem{bertsekas_yu_2009}
D.~P. Bertsekas and H.~Yu.
\newblock Projected equation methods for approximate solution of large linear
  systems.
\newblock {\em Journal of Computational and Applied Mathematics},
  227(1):27–50, 2009.

\bibitem{boyan2002technical}
J.~A. Boyan.
\newblock Technical update: Least-squares temporal difference learning.
\newblock {\em Machine Learning}, 49(2):233--246, 2002.

\bibitem{clopper1934use}
C.~Clopper and E.~S. Pearson.
\newblock The use of confidence or fiducial limits illustrated in the case of
  the binomial.
\newblock {\em Biometrika}, 26(4):404--413, 1934.

\bibitem{cox1979option}
J.~C. Cox, S.~A. Ross, and M.~Rubinstein.
\newblock Option pricing: A simplified approach.
\newblock {\em Journal of financial Economics}, 7(3):229--263, 1979.

\bibitem{gordon1995stable}
G.~J. Gordon.
\newblock Stable function approximation in dynamic programming.
\newblock In {\em Proceedings of the 12th International Conference on Machine
  Learning}, 1995.

\bibitem{iyengar_robust_2005}
G.~N. Iyengar.
\newblock Robust dynamic programming.
\newblock {\em Mathematics of Operations Research}, 30(2):257--280, 2005.

\bibitem{lagoudakis2003least}
M.~G. Lagoudakis and R.~Parr.
\newblock Least-squares policy iteration.
\newblock {\em The Journal of Machine Learning Research}, 4:1107--1149, 2003.

\bibitem{li2009options}
Y.~Li, C.~Szepesvari, and D.~Schuurmans.
\newblock Learning exercise policies for american options.
\newblock In {\em Proc. of the 12th International Conference on Artificial
  Intelligence and Statistics, JMLR: W\&CP}, volume~5, pages 352--359, 2009.

\bibitem{MannorSimesterSunTsitsiklis07}
S.~Mannor, D.~Simester, P.~Sun, and J.~N. Tsitsiklis.
\newblock Bias and variance approximation in value function estimates.
\newblock {\em Management Science}, 53(2):308--322, 2007.

\bibitem{NilimG05}
A.~Nilim and L.~{El Ghaoui}.
\newblock Robust control of {M}arkov decision processes with uncertain
  transition matrices.
\newblock {\em Operations Research}, 53(5):780--798, 2005.

\bibitem{Powell11}
W.~B. Powell.
\newblock {\em Approximate Dynamic Programming}.
\newblock John Wiley and Sons, 2011.

\bibitem{Puterman1994}
M.~L. Puterman.
\newblock {\em Markov decision processes: discrete stochastic dynamic
  programming}.
\newblock John Wiley \& Sons, Inc., 1994.

\bibitem{SuttonBarto98}
R.~S. Sutton and A.~G. Barto.
\newblock {\em Reinforcement Learning: An Introduction}.
\newblock MIT Press, 1998.

\bibitem{sutton_fast_2009}
R.~S. Sutton, H.~R. Maei, D.~Precup, S.~Bhatnagar, D.~Silver, C.~Szepesvári,
  and E.~Wiewiora.
\newblock Fast gradient-descent methods for temporal-difference learning with
  linear function approximation.
\newblock In {\em Proceedings of the 26th Annual International Conference on
  Machine Learning}, 2009.

\bibitem{tsitsiklis2001options}
J.~N. Tsitsiklis and B.~Van~Roy.
\newblock Regression methods for pricing complex american-style options.
\newblock {\em Neural Networks, IEEE Transactions on}, 12(4):694--703, 2001.

\bibitem{yu_convergence_2010}
H.~Yu.
\newblock Convergence of least squares temporal difference methods under
  general conditions.
\newblock In {\em Proceedings of the 27th Annual International Conference on
  Machine Learning}, 2010.

\end{thebibliography}
%\subsubsection*{References}
%
%References follow the acknowledgments. Use unnumbered third level heading for
%the references. Any choice of citation style is acceptable as long as you are
%consistent. It is permissible to reduce the font size to `small' (9-point)
%when listing the references. {\bf Remember that this year you can use
%a ninth page as long as it contains \emph{only} cited references.}
%
%\small{
%[1] Alexander, J.A. \& Mozer, M.C. (1995) Template-based algorithms
%for connectionist rule extraction. In G. Tesauro, D. S. Touretzky
%and T.K. Leen (eds.), {\it Advances in Neural Information Processing
%Systems 7}, pp. 609-616. Cambridge, MA: MIT Press.
%
%[2] Bower, J.M. \& Beeman, D. (1995) {\it The Book of GENESIS: Exploring
%Realistic Neural Models with the GEneral NEural SImulation System.}
%New York: TELOS/Springer-Verlag.
%
%[3] Hasselmo, M.E., Schnell, E. \& Barkai, E. (1995) Dynamics of learning
%and recall at excitatory recurrent synapses and cholinergic modulation
%in rat hippocampal region CA3. {\it Journal of Neuroscience}
%{\bf 15}(7):5249-5262.
%}

\end{document}